\documentclass{article}

\usepackage{amsthm,amsfonts,amsmath,amssymb,epsfig,color,float,graphicx,verbatim}
\usepackage{algorithm,algorithmic}
\usepackage{natbib,fullpage}
\usepackage{hyperref}

\newtheorem{theorem}{Theorem}

\newcommand{\reals}{\mathbb{R}}
\newcommand{\E}{\mathbb{E}}

\newcommand{\bx}{\mathbf{x}}
\newcommand{\bw}{\mathbf{w}}
\newcommand{\bg}{\mathbf{g}}

\newcommand{\Ocal}{\mathcal{O}}

\newcommand{\Wcal}{\mathcal{W}}

\newcommand{\norm}[1]{\|#1\|}
\newcommand{\inner}[1]{\langle#1\rangle}

\newcommand{\secref}[1]{Sec.~\ref{#1}}

\renewcommand{\eqref}[1]{Eq.~(\ref{#1})}

\newcommand{\thmref}[1]{Thm.~\ref{#1}}

\begin{document}
\title{Stochastic Gradient Descent for Non-smooth Optimization: Convergence Results and Optimal Averaging Schemes}
\author{Ohad Shamir\\ohadsh@microsoft.com\\Microsoft Research, One Microsoft Way, Redmond, WA 98052, USA
\date{}
\and
Tong Zhang\\tzhang@stat.rutgers.edu\\Department of Statistics, Rutgers University, Piscataway NJ 08854, USA}

\maketitle

\begin{abstract}
Stochastic Gradient Descent (SGD) is one of the simplest and most popular stochastic optimization methods. While it has already been theoretically studied for decades, the classical analysis usually required non-trivial smoothness assumptions, which do not apply to many modern applications of SGD with non-smooth objective functions such as support vector machines.
In this paper, we investigate the performance of SGD \emph{without} such smoothness assumptions, as well as a running average scheme to convert the SGD iterates to a solution with optimal optimization accuracy. In this framework, we prove that after $T$ rounds, the suboptimality of the \emph{last} SGD iterate scales as $\Ocal(\log(T)/\sqrt{T})$ for non-smooth convex objective functions, and $\Ocal(\log(T)/T)$ in the non-smooth strongly convex case. To the best of our knowledge, these are the first bounds of this kind, and almost match the minimax-optimal rates obtainable by appropriate averaging schemes. We also propose a new and simple averaging scheme, which not only attains optimal rates, but can also be easily computed on-the-fly (in contrast, the suffix averaging scheme proposed in \citet{RakhShaSri12arxiv} is not as simple to implement). Finally, we provide some experimental illustrations.
\end{abstract}

\section{Introduction}

This paper considers one of the simplest and most popular stochastic optimization algorithms, namely Stochastic Gradient Descent (SGD). SGD can be used to optimize any convex function $F$ over a convex domain $\Wcal$, given access only to unbiased estimates of $F$'s gradients (or more generally, subgradients\footnote{Following a common convention, we still refer to the algorithm in this case as ``gradient descent''.}). This feature makes it very useful for learning problems, where our goal is to minimize generalization error based only on a finite sampled training set. Moreover, SGD is extremely simple and highly scalable, making it particularly suitable for large-scale learning problems.

The algorithm itself proceeds in rounds, and can be described in just a few lines: We initialize $\bw_1\in \Wcal$ (following common practice, we will assume $\bw_1=\mathbf{0}$). At round $t=1,2,\ldots$, we obtain a random estimate $\hat{\bg}_t$ of a subgradient $\bg_t\in \partial F(\bw_t)$ so that $\E \hat{\bg}_t = \bg_t$, and update the iterate $\bw_{t}$ as follows:
\[
\bw_{t+1} = \Pi_{\Wcal}(\bw_t-\eta_t \hat{\bg}_t),
\]
where $\eta_t$ is a suitably chosen step-size parameter, and $\Pi_{\Wcal}$ denotes projection on $\Wcal$.

In terms of theoretical analysis, SGD has been studied for decades (for instance, see \citet{KushnerYin03} and references therein), but perhaps surprisingly, there are still important gaps left in our understanding of this method. First of all, most classical results look at \emph{asymptotic} convergence rates, which do not apply to a fixed iteration budget $T$. In recent years, more attention has been devoted to non-asymptotic bounds (e.g., \citet{BachMoulines11}). However, these classical convergence bounds often make non-trivial smoothness assumptions on the function $F$, such as Lipschitz-continuity of the gradient or higher-order derivatives. In modern applications, these assumptions often do not hold. For example, if SGD is used to solve the support-vector machine optimization problem (with the standard non-smooth hinge-loss) on a finite training set, then the underlying objective function $F$ is essentially {\em non-smooth}, even at the optimal solution. In general, for machine learning applications $F$ may be non-smooth whenever one uses a non-smooth loss function, and thus a smoothness-based analysis is not appropriate.

Without assuming smoothness, most of the existing analysis has been carried out in the context of online learning - a more difficult setting than our stochastic setting, where the subgradients are assumed to be provided by an adversary. Using online-to-batch conversion, it is possible to show that after $T$ iterations, the average of the iterates, $(\bw_1+\ldots+\bw_T)/T$, has $\Ocal(\log(T)/T)$ optimization error for strongly-convex $F$ (see precise definition in \secref{sec:preliminaries}), and $\Ocal(1/\sqrt{T})$ error for general convex $F$ \cite{Zin03,HazAgKal07,HazanKa11}. However, \cite{RakhShaSri12arxiv} showed that simple averaging is provably suboptimal in a stochastic setting. Instead, they proposed averaging the last $\alpha T$ iterates of SGD (where $\alpha\in (0,1)$, e.g. $1/2$), and showed that this averaging scheme has an optimal $\Ocal(1/T)$ convergence rate. In comparison, in the non-smooth setting, there are $\Omega(1/\sqrt{T})$ and $\Omega(1/T)$ lower bounds for convex and strongly-convex problems, respectively \cite{agbarawa12}.

These results leave open several issues. First, they pertain to averaging significant parts of the iterates, although in practice averaging just over the last few iterates, or returning the last iterate $\bw_T$, often works quite well (e.g. \cite{ShaSiSreCo11}). Unless $F$ is smooth, the previous results cannot say much about the optimization error of individual iterates. For example, the results in \cite{RakhShaSri12arxiv} only imply an $\Ocal(1/\sqrt{T})$ convergence rate for the last iterate $\bw_T$ with strongly-convex functions, and we are not aware of any results for the last iterate $\bw_T$ in the general convex case. In fact to the best of our knowledge, even for the simpler (non-stochastic) gradient descent method (where $\hat{\bg}_t=\bg_t$), we do not know any existing results that can guarantee the performance of each individual iterate $\bw_T$. Second, the theoretically optimal suffix-averaging scheme proposed in \cite{RakhShaSri12arxiv} has some practical limitations, since it cannot be computed on-the-fly: unless we can store all iterates $\bw_1,\ldots,\bw_T$ in memory, one needs to know the stopping time $T$ beforehand, in order to know when to start computing the suffix average. In practice, $T$ is often not known in advance. This can be partially remedied with a so-called {\em doubling trick}, but it is still not a simple or natural procedure compared to just averaging all iterates, and the latter was shown to be suboptimal in \cite{RakhShaSri12arxiv}.

In this paper, we investigate the convergence rate of SGD and the averaging schemes required to obtain them, with the following contributions:
\begin{itemize}
    \item We prove that the expected optimization error of every individual iterate $\bw_T$ is $\Ocal(\log(T)/T)$ for strongly-convex $F$, and $\Ocal(\log(T)/\sqrt{T})$ for general convex $F$ without smoothness assumptions on $F$. These results show that the suboptimality of the last iterate is not much worse than the optimal rates obtainable by averaging schemes, and partially addresses an open problem posed in \cite{shamir12}. Moreover, the latter result is (to the best of our knowledge) the first finite-sample bound on individual iterates of SGD for non-smooth convex optimization. The proof relies on a technique to reduce results on averages of iterates to results on individual iterates, which was implicitly used in \cite{Zhang04} for a somewhat different setting.
    \item We improve the existing expected error bound on the suffix averaging scheme of \cite{RakhShaSri12arxiv}, from $\Ocal((1+\log(\frac{1}{1-\alpha}))/\alpha T)$ to $\Ocal(\log(\frac{1}{\min\{\alpha,1-\alpha\}})/T)$.
    \item We propose a new and very simple running average scheme, called \emph{polynomial-decay averaging}, and prove that it enjoys optimal rates of convergence. Unlike suffix-averaging, this new running average scheme can be easily computed on-the-fly.
    \item We provide a simple experimental study of the averaging schemes discussed in the paper.
\end{itemize}
We emphasize that although there exist other algorithms with $\Ocal(1/T)$ convergence rate in the strongly convex case (e.g. \cite{HazanKa11,OuGr12}), our focus in this paper is on the simple and widely-used SGD algorithm.


\section{Preliminaries}\label{sec:preliminaries}

We use bold-face letters to denote vectors. We let $F$ denote a convex function over a (closed) convex domain $\Wcal$, which is a subset of some Hilbert space with an induced norm $\norm{\cdot}$. We assume that $F$ is minimized at some $\bw^*\in \Wcal$. Besides general convex $F$, we will also consider the important sub-class of strongly-convex functions. Formally, we say that a function $F$ is \emph{$\lambda$-strongly convex}, if for all $\bw,\bw'\in \Wcal$ and any subgradient $\bg$ of $F$ at $\bw$, it holds that
\[
F(\bw')\geq F(\bw)+\inner{\bg,\bw'-\bw}+\frac{\lambda}{2}\norm{\bw'-\bw}^2 ,
\]
where $\lambda>0$.
For a general convex function, the above inequality can always be satisfied with $\lambda=0$.

As discussed in the introduction, we consider the first-order stochastic optimization setting, where instead of having direct access to $F$, we only have access to an oracle, which given some $\bw\in\Wcal$, returns a random vector $\hat{\bg}$ such that $\E[\hat{\bg}]\in \partial F(\bw)$. Our goal is to use a bounded number $T$ of oracle calls, and compute some $\bar{\bw}\in \Wcal$ such that the optimization error, $F(\bar{\bw})-F(\bw^*)$, is as small as possible. It is well-known that this framework can be applied to learning problems (see for instance \cite{ShalShamSrebSri09b}): given a hypothesis class $\Wcal$ and a set of $T$ i.i.d. examples, we wish to find a predictor $\bw$ whose expected loss $F(\bw)$ is close to optimal over $\Wcal$. Since the examples are chosen i.i.d., the subgradient of the loss function with respect to any individual example can be shown to be an unbiased estimate of a subgradient of $F$. We will mostly consider bounds on the expected error (over the oracle's and algorithm's randomness) for simplicity, although it is possible to obtain high-probability bounds in some cases.

In terms of the step-size $\eta_t$ in the strongly-convex case, we will generally assume it equals $1/(\lambda t)$. We note that this is without much loss of generality, since if the step size is $c/\lambda t$ for some $c\geq 1$, then it is equivalent to taking step sizes $1/(\lambda' t)$ where $\lambda' := \lambda/c \leq \lambda$ is a lower-bound on the strong convexity parameter. Since any $\lambda$-strongly convex function is also $\lambda'$-strongly convex, we can use the analysis here to get upper bounds in terms of $\lambda'$, and if so desired, substitute $\lambda/c$ instead of $\lambda'$ in the final bound.

When we run SGD, we let $\hat{\bg}_t$ denote the random vector obtained at round $t$ (when we query at $\bw_t$), and let $\bg_t=\E[\hat{\bg}_t]$ denote the underlying subgradient of $F$. To facilitate our convergence bounds, we assume that $\E[\norm{\hat{\bg}_t}^2]\leq G^2$ for some fixed $G$. Also, when optimizing general convex functions, we will assume that the diameter of $\Wcal$, namely $\sup_{\bw,\bw'\in\Wcal}\norm{\bw-\bw'}$, is bounded by some constant $D$.

\section{Convergence of Individual SGD Iterates}\label{sec:last}

We begin by considering the case of strongly convex $F$, and prove the following bound on the expected error of any individual iterate $\bw_T$. In this theorem as well as later ones, we did not attempt to optimize constants.
\begin{theorem}\label{thm:strlast}
Suppose $F$ is $\lambda$-strongly convex, and that $\E[\norm{\hat{\bg}_t}^2]\leq G^2$ for all $t$. Consider SGD with step sizes $\eta_t=1/\lambda t$. Then for any $T>1$, it holds that
\[
\E[F(\bw_T)-F(\bw^*)] \leq \frac{17 G^2(1+\log(T))}{\lambda T}.
\]
\end{theorem}
\begin{proof}
The beginning of the proof is standard. By convexity of $\Wcal$, we have the following for any $\bw\in\Wcal$:
\begin{align*}
    &\E\left[\norm{\bw_{t+1}-\bw}^2\right] ~=~ \E[\norm{\Pi_{\Wcal}(\bw_{t}-\eta_t \hat{\bg}_t)-\bw}^2]\\
    &\leq~ \E\left[\norm{\bw_{t}-\eta_t \hat{\bg}_t-\bw}^2\right]\\
    &\leq~ \E\left[\norm{\bw_t-\bw}^2\right]-2\eta_t \E[\inner{\bg_t,\bw_t-\bw}]+\eta_t^2G^2.
\end{align*}
Let $k$ be an arbitrary element in $\{1,\ldots,\lfloor T/2 \rfloor\}$. Extracting the inner product, summing over all $t=T-k,\ldots,T$, and rearranging, we get
\begin{align}
&\sum_{t=T-k}^{T}\E[\inner{\bg_t,\bw_t-\bw}]
~\leq~ \frac{1}{2\eta_{T-k}}\E[\norm{\bw_{T-k}-\bw}^2]\notag\\
&+\sum_{t=T-k+1}^{T}\frac{\E[\norm{\bw_t-\bw}^2]}{2}\left(\frac{1}{\eta_t}-\frac{1}{\eta_{t-1}}\right)+\frac{G^2}{2} \sum_{t=T-k}^{T}\eta_t.\label{eq:beforecrucial}
\end{align}
By convexity of $F$, we can lower bound $\inner{\bg_t,\bw_t-\bw}$ by $F(\bw_t)-F(\bw)$. Plugging this in and substituting $\eta_t = 1/\lambda t$, we get
\begin{align}
&\E\left[\sum_{t=T-k}^{T} (F(\bw_t)-F(\bw))\right]
\leq \frac{\lambda(T-k)}{2}\E[\norm{\bw_{T-k}-\bw}^2]\notag\\
&+\frac{\lambda}{2}\sum_{t=T-k+1}^{T}\E[\norm{\bw_t-\bw}^2]
+\frac{G^2}{2\lambda}\sum_{t=T-k}^{T}\frac{1}{t}.\label{eq:subss}
\end{align}
Now comes the crucial trick: instead of picking $\bw=\bw^{*}$, as done in  standard analysis (\cite{HazAgKal07,RakhShaSri12arxiv}), we instead pick $\bw=\bw_{T-k}$. We also use the fact that $\E\left[\norm{\bw_t-\bw^{*}}^2\right] \leq \frac{4G^2}{\lambda^2 t}$ (\cite{RakhShaSri12arxiv}, Lemma 1), which implies that for any $t\geq T-k$,
\begin{align*}
&\E[\norm{\bw_t-\bw_{T-k}}^2] \\
\leq& 2\E\left[\norm{\bw_t-\bw^*}^2+\norm{\bw_{T-k}-\bw^*}^2\right]\\
\leq&~ \frac{8 G^2}{\lambda^2}\left(\frac{1}{t}+\frac{1}{T-k}\right)
~\leq~ \frac{16 G^2}{\lambda^2(T-k)}
~\leq~ \frac{32 G^2}{\lambda^2 T} .
\end{align*}
Plugging this back into \eqref{eq:subss}, we get
\[
\E\left[\sum_{t=T-k}^{T}(F(\bw_t)-F(\bw_{T-k}))\right]\leq\frac{16 G^2k}{\lambda T}
+\frac{G^2}{2\lambda}\sum_{t=T-k}^{T}\frac{1}{t}.
\]
Let $S_k = \frac{1}{k+1}\sum_{t=T-k}^{T}\E[F(\bw_t)]$ be the expected average value of the last $k+1$ iterates. The bound above implies that
\[
- \E[F(\bw_{T-k})] \leq -\E[S_k]+\frac{G^2}{2\lambda}\left(\frac{32}{T}+\sum_{t=T-k}^{T}\frac{1}{(k+1)t}\right).
\]
By the definition of $S_k$ and the inequality above, we have
\begin{align*}
&k \E[S_{k-1}] = (k+1) \E[S_k] - \E[F(\bw_{T-k})] \\
&\leq~
(k+1)\E[S_k] - \E[S_k]+\frac{G^2}{2\lambda}\left(\frac{32}{ T}+\sum_{t=T-k}^{T}\frac{1}{(k+1)t}\right),
\end{align*}
and dividing by $k$, implies
\begin{equation}\label{eq:srecurse}
\E[S_{k-1}] \leq \E[S_k] + \frac{G^2}{2\lambda}\left(\frac{32}{kT}+\sum_{t=T-k}^{T}\frac{1}{k(k+1)t}\right).
\end{equation}
Using this inequality repeatedly and summing from $k=1$ to $k=\lfloor T/2 \rfloor$, we have
\begin{align}
&\E[F(\bw_T)] = \E[S_0] \leq \E[S_{\lfloor T/2 \rfloor}]+\frac{16 G^2}{\lambda T}\sum_{k=1}^{\lfloor T/2 \rfloor}\frac{1}{k} \nonumber\\
&~+\frac{G^2}{2\lambda}\sum_{k=1}^{\lfloor T/2 \rfloor}\sum_{t=T-k}^{T}\frac{1}{k(k+1)t}. \label{eq:s0-bound}
\end{align}
It now just remains to bound these terms. $\E[S_{T/2}]$ is the expected average value of the last $\lfloor T/2 \rfloor$ iterates, which was already analyzed in (\cite{RakhShaSri12arxiv}, Theorem 5), yielding a bound of
\[
\E [S_{\lfloor T/2 \rfloor}] \leq F(\bw^*)+ \frac{10 G^2}{\lambda T}
\]
for $T>1$. Moreover, we have $\sum_{k=1}^{\lfloor T/2 \rfloor}(1/k) \leq 1+\log(T/2)$.
Finally, we have
\begin{align*}
&\sum_{k=1}^{\lfloor T/2 \rfloor}\sum_{t=T-k}^{T} \frac{1}{k(k+1)t}
\leq \sum_{k=1}^{\lfloor T/2 \rfloor} \frac{1}{k(T-k)} \\
=& \frac{1}{T}\sum_{k=1}^{\lfloor T/2 \rfloor}(\frac{1}{k}+\frac{1}{T-k})
\leq  (1+\log(T))/T.
\end{align*}
The result follows by substituting the above bounds into \eqref{eq:s0-bound} and simplifying for readability.
\end{proof}

Using a similar technique, we can also get an individual iterate bound, in the case of a general convex function $F$ that may be non-smooth. We note that a similar technique was used in \cite{Zhang04}, but for a different algorithm (one with constant learning rate), and the result was less explicit.
\begin{theorem}\label{thm:convlast}
Suppose that $F$ is convex, and that for some constants $D,G$, it holds that $\E[\norm{\hat{\bg}_t}]\leq G^2$ for all $t$, and $\sup_{\bw,\bw'\in\Wcal}\norm{\bw-\bw'}\leq D$. Consider SGD with step sizes $\eta_t=c/\sqrt{t}$ where $c>0$ is a constant. Then for any $T>1$, it holds that
\[
\E[F(\bw_T)-F(\bw^*)] \leq \left(\frac{D^2}{c}+cG^2\right)\frac{2+\log(T)}{\sqrt{T}}.
\]
\end{theorem}
\begin{proof}
The proof begins the same as in \thmref{thm:strlast} (this time letting $k$ be an element in $\{1,\ldots,T-1\}$), up to \eqref{eq:beforecrucial}.
Instead of substituting $\eta_t = c/\lambda t$, we substitute $\eta_t = c/\sqrt{t}$, to get the, $\E[\norm{\bw_t-\bw}^2]$ by $D^2$, pick $\bw=\bw_{T-k}$ and slightly simplify to get
\begin{align*}
&\E\left[\inner{\bg_t,\bw_t-\bw_{T-k}}\right]\\
&~\leq \frac{D^2}{2c}\left(\sqrt{T}-\sqrt{T-k}\right)
+\frac{G^2}{2}\sum_{t=T-k}^{T}\frac{c}{\sqrt{t}}.
\end{align*}
By convexity, we can lower bound $\inner{\bg_t,\bw_t-\bw_{T-k}}$ by $F(\bw_t)-F(\bw_{T-k})$. Also, it is easy to verify (e.g. by integration) that $\sum_{t=T-k}^{T}\frac{1}{\sqrt{t}} \leq 2(\sqrt{T}-\sqrt{T-k-1})$, hence
\begin{align}
& \E\left[\sum_{t=T-k}^{T} (F(\bw_t)-F(\bw_{T-k})) \right] \notag \\
&\leq \left(\frac{D^2}{2c}+cG^2\right)\left(\sqrt{T}-\sqrt{T-k-1}\right)\notag\\
&= \left(\frac{D^2}{2c}+cG^2\right)\frac{k+1}{\sqrt{T}+\sqrt{T-k-1}}\notag\\
&\leq \left(\frac{D^2}{2c}+cG^2\right)\frac{k+1}{\sqrt{T}}.\label{eq:convbb}
\end{align}
As in the proof of \thmref{thm:strlast}, let $S_k = \frac{1}{k+1}\sum_{t=T-k}^{T}\E[F(\bw_t)]$ be the expected average value of the last $K+1$ iterates. The bound above implies that
\[
- \E[F(\bw_{T-k})] \leq -\E[S_k]+\frac{D^2/2c+cG^2}{\sqrt{T}}.
\]
By the definition of $S_k$ and the inequality above, we have
\begin{align*}
&k \E[S_{k-1}] = (k+1) \E[S_k] - \E[F(\bw_{T-k})] \\
&\leq~
(k+1)\E[S_k] - \E[S_k]+\frac{D^2/2c+cG^2}{\sqrt{T}},
\end{align*}
and dividing by $k$, implies
\[
\E[S_{k-1}] \leq \E[S_k] +\frac{D^2/2c+cG^2}{k\sqrt{T}}.
\]
Using this inequality repeatedly and by summing over $k=1,\ldots,T-1$, we have
\begin{equation}
\E[F(\bw_T)] = \E[S_0] \leq \E[S_{T-1}]+\frac{D^2/2c+cG^2}{\sqrt{T}}\sum_{k=1}^{T-1}\frac{1}{k}.\label{eq:convbbb}
\end{equation}
It now just remains to bound the terms on the right hand side. Using \eqref{eq:beforecrucial} with $k=T-1$ and $\bw=\bw^*$, and upper bounding the norms by $D$, it is easy to calculate that
\begin{align*}
&\E[S_{T-1}] - F(\bw^*) = \frac{1}{T}\E\left[\sum_{t=1}^{T}\E[F(\bw_t)-F(\bw^*)\right]\\
&\leq \left(\frac{D^2}{c}+cG^2\right)\frac{1}{\sqrt{T}}.
\end{align*}
Also, we have $\sum_{k=1}^{T-1}1/k \leq (1+\log(T))$. Plugging these upper bounds into \eqref{eq:convbbb} and simplifying for readability, we get the required bound.
\end{proof}

\section{Averaging Schemes}\label{sec:averaging}

The bounds shown in the previous section imply that individual iterates $\bw_T$ have $\Ocal(\log(T)/T)$ expected error in the strongly convex case, and $\Ocal(\log(T)/\sqrt{T})$ expected error in the convex case. These bounds are close but not the same as the minimax optimal rates, which are $\Ocal(1/T)$ and $\Ocal(1/\sqrt{T})$ respectively. In this section, we consider averaging schemes, which rather than return individual iterates, return some weighted combination of all iterates $\bw_1,\ldots,\bw_T$, attaining the minimax optimal rates. We mainly focus here on the strongly-convex case, since simple averaging of all iterates is already known to be optimal (up to constants) in the general convex case.

We first examine the case of \emph{$\alpha$-suffix averaging}, defined as the average of the last $\alpha T$ iterates (where $\alpha \in (0,1)$ is a constant, and $\alpha T$ is assumed to be an integer):
\[
\bar{\bw}_{T}^{\alpha} = \frac{1}{\alpha T}\sum_{t=(1-\alpha)T+1}^T \bw_t.
\]
In \cite{RakhShaSri12arxiv}, it was shown that this averaging scheme results in an optimization error of $\Ocal((1+\log(\frac{1}{1-\alpha}))/\alpha T)$, which is optimal in terms of $T$, but increases rapidly as we make $\alpha$ smaller. The following theorem shows a tighter upper bound of $\Ocal(\log(\frac{1}{\min\{\alpha,1-\alpha\}})/T)$, which implies we can be much more flexible in choosing $\alpha$. Besides being of independent interest, we will re-use this result in our proofs later on.
\begin{theorem}\label{thm:alpha}
Under the conditions of \thmref{thm:strlast}, and assuming $\alpha T$ is an integer, it holds that $\E[F(\bar{\bw}_T^\alpha)-F(\bw^*)]$ is at most
\[
\frac{17 G^2\left(1+\log\left(\frac{1}{\min\{\alpha,(1+1/T)-\alpha\}}\right)\right)}{\lambda T}.
\]
\end{theorem}
\begin{proof}
Suppose first that $\alpha T \leq \lfloor T/2 \rfloor$. The proof is mostly identical to that of \thmref{thm:strlast}, except that instead of using \eqref{eq:srecurse} to bound $\E[S_0]$, we use it to bound $\E[S_{\alpha T-1}] = \frac{1}{\alpha T}\sum_{t=(1-\alpha)T+1}^{T}F(\bw_t)$, which by convexity upper bounds $F(\bar{\bw}_T^\alpha)$. We get:
\begin{align*}
&\E[S_{\alpha T-1}] \leq \E[S_{\lfloor T/2 \rfloor}]+\frac{16 G^2}{\lambda T}\sum_{k=\alpha T}^{\lfloor T/2 \rfloor}\frac{1}{k}\\
&~+\frac{G^2}{2\lambda}\sum_{k=\alpha T}^{\lfloor T/2 \rfloor}\sum_{t=T-k}^{T}\frac{1}{k(k+1)t},
\end{align*}
Using the same argument as in the proof of \thmref{thm:strlast}, and the fact that $\sum_{k=\alpha T }^{\beta T}\frac{1}{k} \leq 1+\log(\beta/\alpha)$ for any integers $\alpha T,\beta T$ that are no larger than $T$, we can obtain the upper bounds
\begin{align*}
\E[S_{\lfloor T/2 \rfloor} \leq& F(\bw^*)+10 G^2/\lambda T \\
\sum_{k=\alpha T}^{\lfloor T/2 \rfloor}\frac{1}{k} \leq & 1+\log(1/2\alpha)
\end{align*}
and
\begin{align*}
& \sum_{k=\alpha T}^{\lfloor T/2 \rfloor}\sum_{t=T-k}^{T}\frac{1}{k(k+1)t}
~\leq~ \frac{1}{T}\sum_{k=\alpha T}^{\lfloor T/2 \rfloor}(\frac{1}{k}+\frac{1}{T-k}) \\
\leq& \frac{1}{T}\left((1+\log(1/2\alpha))+(1+\log(2(1-\alpha)))\right) \\
\leq& \frac{1}{T}\left(2+\log(1/\alpha)\right) .
\end{align*}
Using the above estimates, with some simplifications for readability, we get that $\E[F(\bar{\bw}_T^\alpha)-F(\bw^*)]$ is at most
\begin{equation}\label{eq:upp1}
17\left(1+\log\left(\frac{1}{\alpha}\right)\right)\frac{G^2}{\lambda T}.
\end{equation}
This analysis assumed $\alpha T \leq \lfloor T/2 \rfloor$. If $\alpha$ is larger, we can use the existing analysis (\cite{RakhShaSri12arxiv}, Theorem 5), and get that $\E[F(\bar{\bw}_T^\alpha)-F(\bw^*)]$ is at most
\begin{equation}\label{eq:upp2}
\left(4+5\log\left(\frac{1}{1+1/T -\alpha}\right)\right)\frac{G^2}{\lambda T}.
\end{equation}
Combining \eqref{eq:upp1} and \eqref{eq:upp2} with a uniform upper bound which holds for all $\alpha$, we get the required bound.
\end{proof}

We note that in the general convex case without assuming strong convexity,
one can use an analogous proof to show an upper bound of order $\log(1/\alpha)/\sqrt{T}$ for $\alpha$-suffix averaging. In contrast, existing techniques only imply a bound of order $1/\sqrt{\alpha T}$.

As discussed in the introduction, a limitation of suffix averaging is that unless we can store all iterates in memory, it requires us to guess the stopping time $T$ in advance. For example, if we do $1/2$-suffix averaging, we need to ``know'' when we got to iterate $T/2$ and should start averaging. In practice, the stopping time $T$ is often not known in advance and is determined empirically (e.g. till satisfactory performance is obtained). One way to handle this is to decide in advance on a fixed schedule of stopping times $T$ (e.g. $T_0,2T_0,2^2T_0,2^3T_0,\ldots$ for some $T_0$) and maintain suffix-averages only for those times. However, this is still not very flexible. In contrast, maintaining the average of all iterates up to time $t$ can be done on-the-fly: we initialize $\bar{\bw}_1=\bw_1$, and for any $t>1$, we let
\begin{equation}\label{eq:stdav}
\bar{\bw}_t = \left(1-\frac{1}{t}\right)\bar{\bw}_{t-1}+\frac{1}{t}\bw_t.
\end{equation}
Unfortunately, returning the average of all iterates as in \eqref{eq:stdav}
is provably suboptimal and can harm performance \cite{RakhShaSri12arxiv}. Alternatively, we can easily maintain and return the current iterate $\bw_t$, but we only have a suboptimal $\Ocal(\log(t)/t)$ bound for it.

In the following, we analyze a new and very simple running average scheme, denoted as \emph{polynomial-decay averaging}, and show that it combines the best of both worlds: it can easily be computed on the fly, and it gives an optimal rate. It is parameterized by a number $\eta \geq 0$, which should be thought of as a small constant (e.g. $\eta=3$), and the procedure is defined as follows: $\bar{\bw}_1^{\eta} = \bw_1$, and for any $t>1$,
\begin{equation}\label{eq:expav}
\bar{\bw}_t^{\eta} = \left(1-\frac{\eta+1}{t+\eta}\right)\bar{\bw}_{t-1}^\eta+\frac{\eta+1}{t+\eta}\bw_t.
\end{equation}
For $\eta=0$, this is exactly standard averaging (see \eqref{eq:stdav}), whereas $\eta>0$ reduces the weight of earlier iterates compared to later ones. Moreover, $\bar{\bw}_t^{\eta}$ can be computed on-the-fly, just as easily as computing a standard average.

We note that after this paper was accepted for publication, a similar averaging scheme was independently proposed and studied in \cite{LaJuSchBach12}. Compared to our method, they consider a slightly different step-size and a specific choice of $\eta=1$, using a more direct proof technique tailored to this case.

An analysis of our averaging scheme is provided in the theorem below.
\begin{theorem}\label{thm:polydec}
Suppose $F$ is $\lambda$-strongly convex, and that $\E[\norm{\hat{\bg}_t}^2]\leq G^2$ for all $t$. Consider SGD initialized with $\bw_1$ and step-sizes $\eta_t=1/\lambda t$. Also, let $\eta \geq 1$ be an integer. Then $\E\left[F(\bw_T^{\eta})-F(\bw^*)\right]$ is at most
\[
58\left(1+\frac{\eta}{T}\right)\left(\eta(\eta+1)+\frac{(\eta+0.5)^3(1+\log(T))}{T}\right)\frac{G^2}{\lambda T}
\]
\end{theorem}
The assumption that $\eta$ is an integer is merely for simplicity. Also, we made no effort to optimize the constants, which can be easily improved for specific choices of $\eta$ (see the proof as well as the analysis in \cite{LaJuSchBach12}).
\begin{proof}
We can rewrite the recursion as
\[
\bar{\bw}_t^{\eta}=\frac{t-1}{t+\eta}\bar{\bw}_{t-1}^\eta+\frac{\eta+1}{t+\eta}\bw_t
\]
 for $t \geq 1$ with $\bar{\bw}_0^\eta=0$.
Unwrapping the recursion, we have that for any $T \geq 1$, $\bar{\bw}_T^{\eta} = \sum_{t=1}^{T} \alpha_t \bw_t,$, where
\[
\alpha_t = \frac{\eta+1}{t+\eta}\prod_{j=t+1}^{T}\frac{j-1}{j+\eta} ,
\]
and at $t=T$, the convention that $\prod_{j=T+1}^{T}((j-1)/(j+\eta))=1$ is used.

We now denote $F'(\bw)=F(\bw)-F(\bw^*)$. Since $\bar{\bw}_T^{\eta}$ is a weighted average of $\bw_1,\ldots,\bw_T$, where the weights $\alpha_t$ sum up to be $1$,
it follows by the convexity of $F$ and Jensen's inequality that $F'\left(\bar{\bw}_T^{\eta}\right) \leq \sum_{t=1}^{T}\alpha_t F'(\bw_t)$.

Let $S_k' = \sum_{t=T-k}^{T} F'(\bw_t)$, and $\alpha_0=0$, then we have
\begin{equation}\label{eq:subsfs}
F'\left(\bar{\bw}_T^{\eta}\right) \leq \sum_{t=1}^{T} (\alpha_t-\alpha_{t-1}) S'_{T-t} .
\end{equation}
It is not difficult to check that for all $t \geq 1$:
\begin{align}
\alpha_t - \alpha_{t-1} =& \frac{\eta (\eta+1)}{(t-1+\eta)(t+\eta)} \prod_{j=t+1}^{T}\frac{j-1}{j+\eta} \notag\\
= & \frac{\eta (\eta+1)}{T(T+1)} \prod_{j=t}^{T+1}\frac{j}{j-1+\eta} \notag\\
\leq &
\begin{cases}
\frac{\eta (\eta+1)}{T(T+1)} \left( \frac{t-2+\eta}{T+\eta} \right)^{\eta-1} & \text{if } t \leq T+2 -\eta \\
\frac{\eta (\eta+1)}{T(T+1)}& \text{otherwise}
\end{cases}.\label{eq:twocases}
\end{align}
Let us suppose first that $\eta \geq 2$. In that case, we can upper bound the above by
\[
\frac{\eta (\eta+1) (t+\eta)}{T(T+1)(T+2)}.
\]
As to $S'_{T-t}$ in \eqref{eq:subsfs}, note that the upper bound proof of \thmref{thm:alpha} equally applies to $\frac{1}{T-t+1}S'_{T-t}$. Using this bound and substituting in \eqref{eq:subsfs}, we obtain
\begin{align}
F'&\left(\bar{\bw}_T^{\eta}\right)\notag\\
\leq& \sum_{t=1}^T (\alpha_t-\alpha_{t-1})
(T-t+1) \frac{17 G^2 \log\left(\frac{T e}{\min\{t,T-t+1\}}\right)}{\lambda T} \notag\\
\leq & \sum_{t=1}^{\lceil T/2\rceil} \frac{2\eta (\eta+1)(t+\eta)}{T(T+1)(T+2)}(T+\eta)
\frac{17 G^2 \log\left(T e/t\right)}{\lambda T} \notag\\
\leq & \frac{34 G^2 \eta (\eta+1) (T+\eta)}{\lambda T^2(T+1)(T+2)}  (A + B + C) , \label{eq:ABC}
\end{align}
where
\[
A = \sum_{t=1}^{\lceil T/2\rceil} \eta \log(Te/t)
\leq \eta \frac{T+1}{2} \log( T e) ,
\]
and
\[
B = \sum_{t=1}^{\lceil T/2\rceil} t \log(Te)
\leq 0.5 (\lceil T/2 \rceil)(\lceil T/2 \rceil +1) \log(T e) ,
\]
and
\begin{align*}
C\leq&- \sum_{t=1}^{\lceil T/2\rceil} t \log(t)
\leq - \int_{t=1}^{\lceil T/2\rceil} t \log(t) d t \\
=&
- \left[ 0.5 t^2 \log t - 0.25 t^2 \right] \big|_{1}^{\lceil T/2 \rceil} \\
=&
- 0.5 \lceil T/2\rceil^2 \log (T/2)
+ 0.25 \lceil T/2 \rceil^2 - 0.25 .
\end{align*}
Therefore we have
\begin{align*}
&A+ B + C  \\
\leq &
(\eta+0.5) \frac{T+1}{2} \log( T e) \\
& \quad + 0.5 \frac{(T+1)^2}{4}\log( 2 e^{1.5}) -0.25 \\
\leq & (\eta+0.5) \frac{T+1}{2} \log( T e) + 0.5 (T+1)(T+2) .
\end{align*}
Plugging this estimate into \eqref{eq:ABC} and simplifying, we obtain an upper bound on $\E\left[F(\bw_T^{\eta})-F(\bw^*)\right]$ of the form
\begin{equation}\label{eq:eta2}
17\left(1+\frac{\eta}{T}\right)\left(\eta(\eta+1)+\frac{(\eta+0.5)^3(1+\log(T))}{T}\right)\frac{G^2}{\lambda T}.
\end{equation}
It remains to treat the case $\eta=1$. In that case, the upper bound on $\alpha_t-\alpha_{t-1}$ in \eqref{eq:twocases} becomes
\[
\alpha_{t}-\alpha_{t-1} \leq \frac{2}{T(T+1)},
\]
and using the same derivation as before, we get that
\[
F'\left(\bar{\bw}_T^{\eta}\right) \leq \sum_{t=1}^{\lceil T/2\rceil} \frac{68 G^2 \log(Te/t)}{\lambda T^2}.
\]
By an integration calculation, it is easy to verify that
\begin{align*}
&\sum_{t=1}^{\lceil T/2\rceil}\log(Te/t) \leq \left\lceil \frac{T}{2} \right\rceil \log(Te)-\int_{t=1}^{\lceil T/2 \rceil} \log(t) dt \\
&= \left\lceil \frac{T}{2} \right\rceil \log(Te)-\left[ t \log(t)-t\right]\big|_{1}^{\lceil T/2 \rceil}\\
&\leq \left\lceil \frac{T}{2} \right\rceil \left(2+\log(2)\right)-1.
\end{align*}
Plugging it in, we get an upper bound of
\[
\frac{68G^2}{\lambda T}\frac{\lceil T/2 \rceil(2+\log(2))-1}{T},
\]
which for any $T\geq 1$ is at most $116G^2/\lambda T$. The stated result follows by combining this bound (for $\eta=1$) and \eqref{eq:eta2} (for $\eta\geq 2$), increasing the numerical constant in \eqref{eq:eta2} to obtain a uniform bound which holds for all choices of $\eta$.
\end{proof}
Note that for a constant $\eta$, the bound is essentially optimal.
We end by noting that using an identical proof technique, it holds in the case of general convex $F$ (with assumptions similar to \thmref{thm:convlast}) that
\[
\E\left[F(\bw_T^{\eta})-F(\bw^*)\right] \leq \Ocal\left(\frac{\eta(D^2/c+cG^2)}{\sqrt{T}}\right),
\]
this implies that polynomial-decay averaging is also optimal (up to constants) in the general convex case.

\section{Experiments}\label{sec:experiments}

In this section, we study the behavior of the polynomial-decay averaging scheme on a few strongly-convex optimization problems. We chose the same $3$ binary classification datasets ((\textsc{ccat},\textsc{cov1} and \textsc{astro-ph}) and experimental setup as in \cite{RakhShaSri12arxiv}. For each dataset $\{\bx_i,y_i\}_{i=1}^{m}$, we ran SGD on the support vector machine optimization problem
\[
F(\bw) = \frac{\lambda}{2}\norm{\bw}^2 + \frac{1}{m}\sum_{i=1}^{m}\max\{0,1-y_i\inner{\bx_i,\bw}\},
\]
with the domain $\Wcal=\reals^d$, where the stochastic gradient given $\bw_t$ was computed by taking a single randomly drawn training example $(\bx_i,y_i)$ and computing the gradient with respect to that example, i.e. $\hat{\bg}_t = \lambda \bw_t - \mathbf{1}_{y_i\inner{\bx_i,\bw_t}\leq 1}y_i \bx_i$. All algorithms were initialized at $\bw_1=0$. Following previous work, we chose $\lambda=10^{-4}$ for \textsc{ccat}, $\lambda=10^{-6}$ for \textsc{cov1}, and $\lambda=5\times 10^{-5}$ for \textsc{astro-ph}. The $\eta$ parameter of polynomial-decay averaging was set to $3$. For comparison, besides polynomial-decay averaging, we also ran suffix averaging with $\alpha=1/2$, and simple averaging of all iterates. The results are reported in the figure below. Each graph is a log-log plot representing the training error on one dataset over $10$ repetitions, as a function of the number of iterations. We also experimented on the test set provided with each dataset, but omit the results as they are very similar.

The graphs below clearly indicate that polynomial-decay averaging work quite well. Achieving the best or almost-best performance in all cases. Suffix averaging performs  performs similarly, although as discussed earlier, it is not as amenable to on-the-fly computation. Compared to these schemes, a simple average of all iterates is significantly suboptimal, matching the results of \cite{RakhShaSri12arxiv}.

\section{Discussion}\label{sec:discussion}

In this paper, we investigated the convergence behavior of SGD, and the averaging schemes required to obtain optimal performance. In particular, we considered polynomial-decay averaging, which is as simple to compute as standard averaging of all iterates, but attains better performance theoretically and in practice. We also extended the existing analysis of SGD by providing new finite-sample bounds on individual SGD iterates, which hold without any smoothness assumptions, for both convex and strongly-convex problems.
Finally, we provided new bounds for suffix averaging. While we focused on standard gradient descent, our techniques can be extended to the more general mirror descent framework and non-Euclidean norms.

An important open question is whether the $\Ocal(\log(T)/T)$ rate we obtained on the individual iterate $\bw_T$, for strongly-convex problems, is tight.
\begin{figure}[!ht]
\begin{center}
\includegraphics[trim = 0cm 0.5cm 0.5cm 0.5cm, clip=true, scale=1]{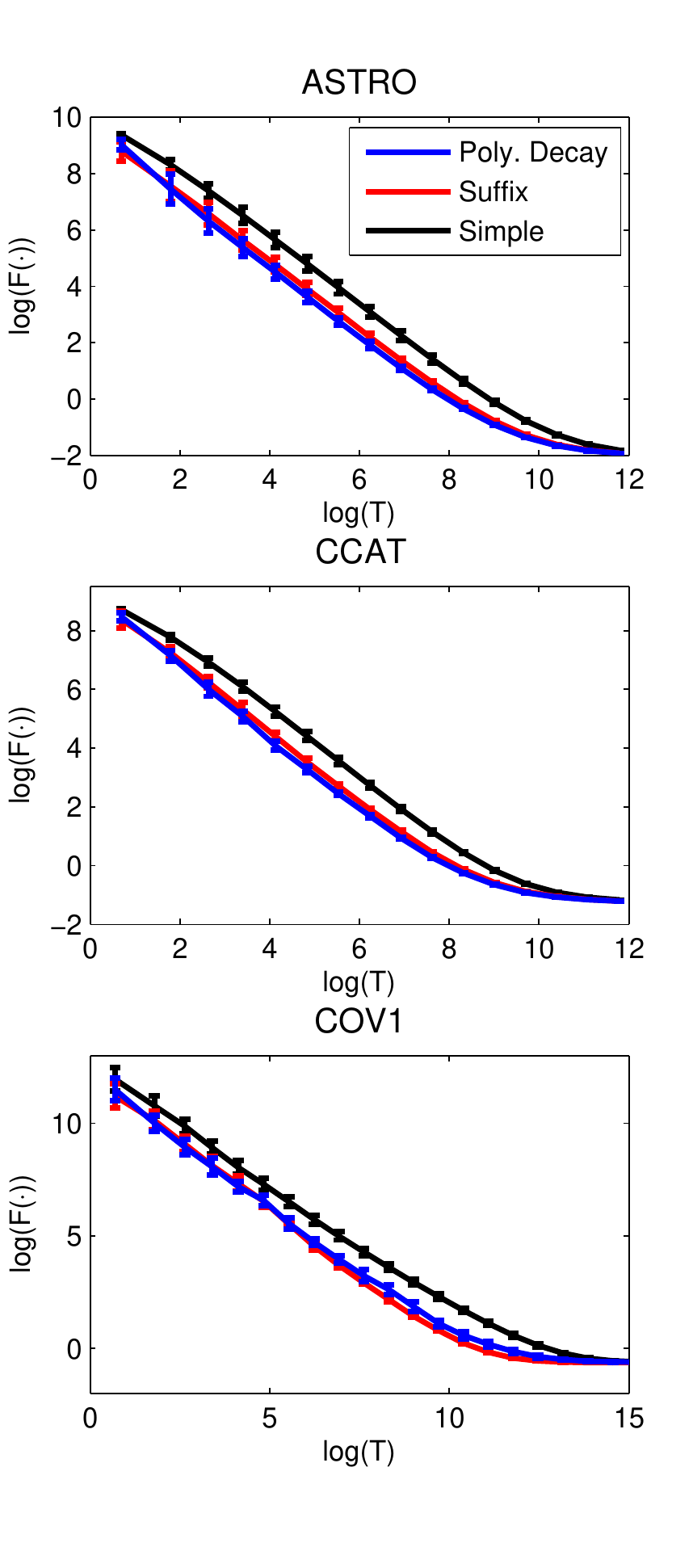}
\end{center}
\vskip -0.35in
\end{figure}
This question is important, because running SGD for $T$ iterations, and returning the last iterate $\bw_T$, is a very common heuristic. If the $\Ocal(\log(T)/T)$ bound is tight, it means practitioners should \emph{not} return the last iterate, since better $\Ocal(1/T)$ rates can be obtained by suffix averaging or polynomial-decay averaging. Alternatively, a $\Ocal(1/T)$ bound on the last iterate can indicate that returning the last iterate is indeed justified. For a further discussion of this, see \cite{shamir12}. Another question is whether high-probability versions of our individual iterate bounds (\thmref{thm:strlast} and \thmref{thm:convlast}) can be obtained, especially in the strongly-convex case. Again, this question has practical implications, since if a high-probability bound does not hold, it might imply that the last iterate can suffer from high variability, and should be used with caution. Finally, the tightness of \thmref{thm:convlast} is still unclear. In fact, even for the simpler case of (non-stocahstic) gradient descent, we do not know whether the behavior of the last iterate proved in \thmref{thm:convlast} is tight. In general, for an algorithm as simple and popular as SGD, we should have a better understanding of how it behaves and how it should be used in an optimal way.

\textbf{Acknowledgements:} We thank Simon Lacoste-Julien for helpful comments.

\bibliography{mybib}
\bibliographystyle{icml2013}

\end{document}